\documentclass[]{nesy2026} 

\usepackage{longtable}
\usepackage{booktabs}
\usepackage[load-configurations=version-1]{siunitx}
\usepackage{mathtools}
\usepackage{bbm}
\usepackage{float}
\usepackage{multirow}
\usepackage{paralist}
\usepackage{todonotes}
\usepackage{enumitem}
\usepackage[font=small,skip=3pt]{caption}
\usepackage{multicol}
\usepackage{amsfonts}
\usepackage{amsmath}
\usepackage{stmaryrd}
\usepackage{xspace}
\usepackage{multirow}
\usepackage{comment}
\usepackage{booktabs}

\setlist{nosep,leftmargin=*}
\small
\renewcommand{\arraystretch}{0.9}
\theorembodyfont{\upshape}
\theoremheaderfont{\scshape}
\theorempostheader{:}
\theoremsep{\newline}

\newcommand{\ApM}{\mathrm{A_p^M}}
\newcommand{\ApME}{\mathrm{A_p^{ME}}}
\newcommand{\X}{\mathbf{X}}
\newcommand{\WX}{\mathbf{WX}}
\newcommand{\F}{\mathbf{F}}
\newcommand{\G}{\mathbf{G}}
\newcommand{\U}{\mathbf{U}}
\newcommand{\R}{\mathbf{R}}
\DeclareMathOperator{\agg}{Agg}
\usepackage{amsmath}
\DeclareMathOperator*{\argmax}{arg\,max}

\newif\ifsolutions
\solutionstrue
\solutionsfalse

\ifsolutions
  \includecomment{solution}
\else
  \excludecomment{solution}
\fi

\newcommand{\nbb}[1]{\hbox to 0pt{\textcolor{red}{\bf!}}%
\marginpar
{\parbox{20mm}{\scriptsize\textcolor{red}{#1}}}}

\newcommand{\p}{\varphi}

\newcommand{\mdl}{\models}
\newcommand{\Int}{\ensuremath{\mathcal{I}}\xspace}
\newcommand{\Tmf}{\ensuremath{\mathfrak{T}}\xspace}
\newcommand{\var}{\ensuremath{\mathcal{X}}\xspace}

\newcommand{\assign}{\ensuremath{\mathfrak{a}}\xspace}

\newcommand{\Until}{\ensuremath{\mathcal{U}}\xspace}

\newcommand{\Release}{\ensuremath{\mathbin{\mathcal{R}}}\xspace}

\title[First-Order Temporal Logic Tensor Networks]{First-Order Temporal Logic Tensor Networks}

\clearauthor{
  \Name{Luca Boscarato} \Email{lboscarato@unibz.it}\\
  \addr Free University of Bozen-Bolzano
  \AND
  \Name{Ivan Donadello} \Email{ivan.donadello@unibz.it}\\
  \addr Free University of Bozen-Bolzano
  \AND
  \Name{Alessandro Artale} \Email{Alessandro.Artale@unibz.it}\\
  \addr Free University of Bozen-Bolzano
  \AND
  \Name{Marco Montali} \Email{marco.montali@unibz.it}\\
  \addr Free University of Bozen-Bolzano
  \AND
  \Name{Fabrizio Maria Maggi} \Email{maggi@unibz.it}\\
  \addr Free University of Bozen-Bolzano}

\begin{document}
\maketitle
\thispagestyle{empty}   
\pagestyle{plain}       

\begin{abstract}
\begin{solution}
\textcolor{blue{
We present \emph{Linear Temporal Real Logic over Finite Traces} (LTRL\textsubscript{f}), an extension of Real Logic---the differentiable, many-valued semantics underlying Logic Tensor Networks---to the temporal domain.
LTRL\textsubscript{f} augments first-order Real Logic with the usual linear temporal operators (Next, Until, Eventually, Globally, Release), interpreted over finite discrete traces, following a fuzzy LTL framework.
Temporal operators are implemented through differentiable approximations (generalized means) of fuzzy supremum and infimum, ensuring full compatibility with gradient-based optimization techniques.
In particular, we propose a formal syntax and grounded semantics, prove differentiability of the resulting evaluation, and validate the framework with a reference implementation, including a parser and temporal evaluator built on top of the Logic Tensor Network framework.
Finally, we evaluate the framework against both neurosymbolic and purely neural approaches, as well as presenting some numerical examples to demonstrate that LTRL\textsubscript{f} correctly generalizes Boolean LTL\textsubscript{f} semantics and supports expressive spatio-temporal constraints.}
\end{solution}
Most of the existing neuro-symbolic AI methods focus on the scenario of static knowledge where objects do not change according to a temporal dimension. Temporal neuro-symbolic works are still under explored and are mainly developed for time-interval logic or propositional linear temporal logic. There is a lack of models studying linear temporal logics with predicates that deal with objects whose properties and relations change through the time. We present First-Order Temporal Logic Tensor Networks (FOT-LTN) that is an extension of Logic Tensor Networks (LTN) that fills this gap by considering a linear-temporal dimension. In particular, FOT-LTN joins the syntax of First-Order Linear Temporal Logic with the fuzzy (and real-valued) semantics of LTN obtaining a framework that supports both temporal operators and quantifiers and is totally differentiable. A first evaluation regards a temporal knowledge graph completion task on two synthetic datasets showing better performance of FOT-LTN with respect to dedicated (purely neural) methods.
\end{abstract}

\section{Introduction}\label{sec:intro}
Neuro-Symbolic AI (NeSy) is a prominent paradigm for developing AI systems that are both data-driven and verifiably robust. While pure deep learning architectures excel at pattern recognition, they fundamentally lack the capacity for compositional reasoning and the adherence to explicit domain constraints. Conversely, classical symbolic logic offers rigorous guarantees and explainability but struggles to handle noisy, real-world data.

\begin{solution}
\textcolor{blue}{
Among the frameworks addressing this dichotomy, Logic Tensor Networks (LTNs) have gained significant traction. LTNs bridge the gap by embedding First-Order Logic (FOL) formulas into continuous vector spaces using fuzzy logic semantics (Real Logic). In an LTN, symbols such as constants, functions, and predicates are grounded as tensors, deep neural networks, and differentiable operations, respectively. This allows domain knowledge to be injected directly into the neural network's loss function as differentiable logical constraints, enabling end-to-end gradient-based optimization.}
\end{solution}

Despite the success of several NeSy frameworks in static domains, many real-world applications—ranging from autonomous driving to process monitoring~\citep{Di_Francescomarino2026-uy}—are inherently dynamic. In these settings, data arrives as sequential streams with underlying rules expressed as temporal dependencies. Classical NeSy frameworks based on First-Order Logic, such as Logic Tensor Networks~\citep{SerafiniGBDSB21,serafini_garcez_2020}, are structurally unequipped to express temporal relations, such as ``event A must always happen eventually after event B''. While Linear Temporal Logic (LTL), or its finite-trace variant LTL\textsubscript{f}~\citep{DBLP:conf/ijcai/GiacomoV13}, provides an established syntax for reasoning over time even with a (differentiable) fuzzy semantics~\citep{DBLP:journals/corr/abs-1203-6278}, its standard formulation is propositional and cannot express temporal relations between objects, such as ``every car object must stop until a closing pedestrian object is crossing''. To close this gap, we introduce First-Order Temporal Logic Tensor Networks (FOT-LTN), a novel NeSy framework that extends LTN to linear temporal structures. By defining differentiable temporal operators over finite traces, FOT-LTN allows neural networks to learn from and reason about temporal sequences in presence of temporal constraints with predicates. The key contributions are:
\begin{itemize}
    \item FOT-LTN, a new NeSy framework that learns and reasons with a first-order linear temporal logic expressivity with a real-valued semantics.
    \item A differentiable (real-valued) semantics for the first-order temporal language underlying FOT-LTN, as we formalize the continuous evaluations of temporal operators and quantifiers in a differentiable manner.
    \item A first evaluation on synthetic datasets simulating traffic scenes where FOT-LTN improves the performance (and the logical consistency) in a temporal knowledge graph task over a dedicated, purely neural, method, especially in low-supervision regimes.
\end{itemize}

\begin{solution}
    \textcolor{blue}{
    Mathematical Guarantees: We provide a rigorous inductive proof (detailed in Appendix C) demonstrating that the truth values of LTRLf formulas are differentiable everywhere on the valid domain of neural network activations.}
    \end{solution}

\begin{solution}
\textcolor{blue}{
Neuro-symbolic AI aims at combining deep learning --- which excels in many tasks --- with logic, in order to obtain more interpretable results. In this context, Logic Tensor Networks~\citep{DBLP:conf/nesy/SerafiniG16,serafini_garcez_2020} represent a relevant framework, as they ground first-order logic formulas into differentiable computations over tensors, using \emph{Real Logic}---a many-valued semantics where truth values are constrained in the interval $[0,1]$ and connectives are implemented by fuzzy operators (t-norms, t-conorms, fuzzy negation and implication).
Once the knowledge base, composed of logical axioms, has been defined, the neural network aims at maximizing the aggregate satisfaction of the knowledge base by turning logical constraints into differentiable loss terms. In such a way, it is able to learn representations that satisfy the given logical constraints, enabling neuro-symbolic integration. 
The original LTN framework focuses on \emph{static} constraints and does not natively support reasoning about \emph{temporal sequences}.
However, many real-world applications require expressing temporal dependencies, which are naturally expressed in temporal logics such as LTL or its finite-trace variant LTL\textsubscript{f}~\citep{DBLP:conf/ijcai/GiacomoV13}.
Hence, in this paper, we introduce \emph{Linear Temporal Real Logic over Finite Traces} (LTRL\textsubscript{f}), which extends Real Logic with the usual LTL\textsubscript{f} temporal operators (Next $\X$, Weak Next $\WX$, Eventually $\F$, Globally $\G$, Until $\U$, Release $\R$), interpreted over finite discrete traces using fuzzy, differentiable semantics.
The temporal evaluation follows the fuzzy Linear Temporal Logic (LTL) proposed by \citet{DBLP:journals/corr/abs-1203-6278}, instantiated with differentiable generalized means as used in LTN quantifier aggregation.
Finally, we prove that the resulting grounded semantics is differentiable with respect to model parameters, enabling gradient-based learning of temporal constraints.}
\end{solution}

\section{Background}\label{sec:background}
FOT-LTN merges LTN with First-Order Linear Temporal Logic over finite traces.

\subsection{Real Logic and Logic Tensor Networks}\label{sub:sec:RL}

Logic Tensor Networks (LTN)~\citep{serafini_garcez_2020} is based on a fuzzy first-order language $\mathcal{L}=(\mathcal{C},\mathcal{F},\mathcal{P},\mathcal{X},\delta,\delta_{in}\delta_{out})$ (called Real Logic), typed over a non-empty set $\mathbb{D}$ of \textit{domain symbols}, where $\mathcal{C}$ is a set of constant symbols, $\mathcal{F}$ is a set of function symbols, $\mathcal{P}$ is a set of predicate symbols and $\mathcal{X}$ is a set of variable symbols. The \emph{arity} of a function or predicate is the number of arguments of the function or predicate, respectively.
In order to assign a type to each object, as well as to functions and predicates, we define three typing functions $\delta$, $\delta_{\mathrm{in}}$ and $\delta_{\mathrm{out}}$ as follows:
\begin{itemize}
    \item $\delta: \mathcal{X} \cup \mathcal{C} \to \mathbb{D}$, returns the domain type of a variable or constant;
    \item $\delta_{\mathrm{in}}: \mathcal{F} \cup \mathcal{P} \to \mathbb{D}^n$, for each $n$-ary function or predicate, where $\mathbb{D}^n$ is the $n$-ary cartesian product, returns the domains of the $n$ arguments of functions and predicates;
    \item $\delta_{\mathrm{out}} : \mathcal{F} \to \mathbb{D}$, 
    returns the range of a function symbol.
\end{itemize}
A term $\tau$ is constructed recursively from variables, constants and function symbols. In particular, every term $\tau \in \mathcal{X} \cup \mathcal{C}$ is a term of domain $\delta(\tau)$. If $f \in \mathcal{F}$ is an $n$-ary function symbol and the sequence of terms $(\tau_1,\ldots,\tau_n)$ has domain sequence $\delta_{\mathrm{in}}(f)$, then $f(\tau_1,\ldots,\tau_n)$ is a term with domain $\delta_{\mathrm{out}}(f)$.
Atomic formulas are either equalities $\tau_1=\tau_2$, where the two terms must have the same domain, or n-ary predicates $p(\tau_1,\ldots,\tau_n)$ with domains of $(\tau_1,\ldots,\tau_n)$ matching $\delta_{\mathrm{in}}(p)$.
More complex formulas can be constructed starting from atomic ones by means of fuzzy connectives $\neg,\wedge,\vee,\rightarrow,\leftrightarrow$ and quantifiers $\forall,\exists$.

The semantics of Real Logic differs from classical first-order semantics in that symbols are grounded in real-valued tensors rather than interpreted over abstract domains.
A grounding $\mathcal{G}$ assigns: to each domain symbol $D \in \mathbb{D}$ a set of tensors
\[
\mathcal{G}(D) \subseteq \bigcup_{k \in \mathbb{N}} \ \bigcup_{n_1,\ldots,n_k \in \mathbb{N}} \mathbb{R}^{n_1 \times \cdots \times n_k}.
\]
In particular, given a sequence of domains $D_1,\ldots,D_k$ (with $D_i \in \mathbb{D}$), $\mathcal{G}(D_1,\ldots,D_k) = \bigtimes_{i=1}^k \mathcal{G}(D_i)$. The grounding assigns to each constant symbol $c \in \mathcal{C}$ a tensor $\mathcal{G}(c) \in \mathcal{G}(\delta(c))$, and to each variable $x \in \mathcal{X}$ a finite sequence of tensors $\mathcal{G}(x) = (d_1,\ldots,d_k)$ with $d_i \in \mathcal{G}(\delta(x))$.
\begin{solution}
\nbb{A: Why is $\mathcal{G}(x)$ mapped to a sequence, and not to a set? And why only finitely many elements? Is this common in real logic? Or shall we note this explicitly?}    
\end{solution}
Such finite sequences provide the grounded instances over which quantified formulas are evaluated. Furthermore, each function symbol $f \in \mathcal{F}$ is assigned to a function
$
\mathcal{G}(f): \mathcal{G}(\delta_{\mathrm{in}}(f)) \to \mathcal{G}(\delta_{\mathrm{out}}(f)),
$
and each predicate symbol $p \in \mathcal{P}$ is assigned a function
$
\mathcal{G}(p): \mathcal{G}(\delta_{\mathrm{in}}(p)) \to [0,1].
$
Thus, terms are interpreted as tensors, while formulas are interpreted as truth degrees in the interval $[0,1]$. Logical connectives are defined by means of fuzzy-logic operators, such as t-norms, t-conorms, fuzzy implications, and fuzzy negations. On the other hand, first order quantifiers are expressed via aggregation operators.

\subsection{First Order Linear Temporal Logic over Finite Traces}\label{sec:foltlf}

The First-Order Linear Temporal Logic over finite traces (FO-LTL\textsubscript{f}), is obtained by
extending the usual first-order language with the temporal operator \emph{until}, $\U$, interpreted over linear structures, called \emph{traces}. Here we present FO-LTL\textsubscript{f} as the logic interpreted over \emph{finite linear structures}~\citep{DBLP:conf/ijcai/ArtaleMO19,DBLP:journals/tocl/ArtaleMO24} which, in turns, adapts to the first-order setting the Linear Temporal Logic on finite linear structures, LTL\textsubscript{f}~\citep{DBLP:conf/ijcai/GiacomoV13}. Formulas of FO-LTL\textsubscript{f} are of the form:
$$
\p
::= 
p(\tau_1,\ldots,\tau_n)
\mid
\tau_1 = \tau_2
\mid
\neg \p \mid
(\p \land \p)
\mid \exists x \p \mid
(\p \Until \p).
$$

A \emph{first-order temporal interpretation} (or \emph{trace}) is a
pair $\Int =(\Delta, (\Int_{n})_{n \in \Tmf})$, where $\Tmf$ is a
sub-order of $(\mathbb{N}, <)$ of the form $[0,l]$,
with $l \in \mathbb{N}$, and each $\Int_{n}$ is a classical first-order interpretation with
a non empty domain $\Delta$ such that
$p^{\Int_{n}} \subseteq \Delta^{\textsf{ar}(p)}$, for 
$p$ an atomic predicate of arity ${\textsf{ar}(p)}$, and
$c^{\Int_{i}}=c^{\Int_{j}}\in \Delta$ for all $c\in \mathcal{C}$ and
$i,j \in \mathbb{N}$, i.e., constants are \emph{rigid designators} (with
fixed interpretation, denoted simply by $c^{\Int}$).
The stipulation that all time points share the same domain $\Delta$ is
called the \emph{constant domain assumption} (i.e., objects are
not created or destroyed over time), and it is the most general choice
in the sense that increasing, decreasing, and varying domains can all
be reduced to it~\citep{Gabbay2005-GABMML-2}.
An \emph{assignment} is a function $\assign$ from $\var$ to $\Delta$, and
the \emph{value of a term $\tau$ under $\assign$} is defined
as: $\assign(\tau) = \assign(x)$, if $\tau = x \in \var$, and
$\assign(\tau) = c^{\Int}$, if $\tau = c \in \mathcal{C}$.   
Given a formula~$\p$, the \emph{satisfaction of
  $\p$ in~$\Int$ at time point $n \in \Tmf$ under an assignment
  $\assign$}, written $\Int, n \models^{\assign} \p$, is inductively
defined as:
\[
\begin{array}{lcl}
		\Int, n \models^{\assign} 
		p(\tau_1,\ldots,\tau_n)
		& \text{iff} & 
		(\assign(\tau_1),\ldots,\assign(\tau_n)) \in p^{\Int_n},\\
		\Int, n \models^{\assign} 
		\tau_1 = \tau_2
		& \text{iff} & 
		\assign(\tau_1) = \assign(\tau_n),\\

		\Int, n  \models^{\assign} \neg \psi & \text{iff} & \text{not } \Int, n  \models^{\assign} \psi, \\
		\Int, n  \models^{\assign} \psi \land \chi & \text{iff} & \Int, n  \models^{\assign} \psi \text{ and } \Int, n  \models^{\assign} \chi, \\
		\Int, n  \models^{\assign} \exists x \psi & \text{iff}
                             & \Int, n \models^{\assign'} \psi, \text{
                               for some assignment } \assign' \text{ that can differ from } \assign \text{ only on } x, \\
		\Int, n  \models^{\assign} \psi \Until \chi &
		                                                  		\text{iff} & \text{there is }  m \in \Tmf, m > n  \colon \Int, m \models^{\assign}\chi \text{ and, } \text{for all } i \in (n,m), \Int, i \models^{\assign} \psi. \\

 \end{array}
\]
We say that $\varphi$ is \emph{satisfied in $\Int$}
\emph{under $\assign$}, writing $\Int \models^{\assign} \varphi$, if
$\Int, 0 \models^{\assign} \varphi$, and that
$\varphi$ is \emph{satisfied in $\Int$} (or that $\Int$ is a
\emph{model} of $\varphi$), denoted by $\Int \models \varphi$, if
$\Int \models^{\assign} \varphi$, for some $\assign$.
Moreover, $\p$ is said to be \emph{satisfiable} if it is satisfied in
some $\Int$. A formula $\p$ \emph{logically implies} a formula $\psi$
if, for every interpretation $\Int$ and every assignment $\assign$,
$\Int \models^{\assign} \p$ implies $\Int \models^{\assign} \psi$,
and we write $\p \mdl \psi$.
We say that $\p$ and $\psi$ are \emph{equivalent}, writing $\p \equiv \psi$, if $\p \mdl \psi$ and $\psi \mdl \p$.

In addition to the standard Boolean equivalences, the following equivalences hold:
$\bot \equiv p \land \lnot p$ (\emph{bottom}, for an arbitrary but fixed $0$-ary predicate $p$);
$\top \equiv \lnot \bot$ (\emph{top});
$\p \U \psi \equiv \psi \lor (\p \land \p \Until \psi)$
(\emph{reflexive until});
$\Diamond \p \equiv \top \Until \p$ (\emph{sometime});
$\F \p \equiv \top \U \p$ (\emph{reflexive sometime});
$\Box \p \equiv \lnot \Diamond \lnot \p$ (\emph{always});
$\G \p \equiv \lnot \F \lnot \p$ (\emph{reflexive always});
$\p \Release \psi \equiv \lnot ( \lnot \p \Until \lnot \psi)$
(\emph{releases});
$\p \R \psi \equiv \psi \land (\p \lor \p \Release \psi)$;
(\emph{reflexive releases});
$\X \p \equiv \bot \Until \p$ (\emph{strong next});
$l \equiv \lnot \X \top$ (\emph{last point});
and
$\WX \p \equiv l \lor \X\p$ (\emph{weak next}).
Note that, $l$ is true only on the last point of a finite trace, and the weak next operator, $\WX\p$, is relevant only on finite traces where it is satisfied at a time point iff either it is the last instant of the finite trace, or at the next time point $\p$ holds.

\section{First-Order Temporal Logic Tensor Networks}\label{sec:approach}
FOT-LTN is based on Linear Temporal Real Logic over finite traces (LTRL\textsubscript{f}), a new logic, based on FO-LTL\textsubscript{f}, which extends Real Logic with temporal operators. In LTRL\textsubscript{f}, terms and formulas have the following form:
$$
\begin{aligned}
\tau &::= x \mid c \mid f(\tau_1,\dots,\tau_n) \\
\phi &::= \top \mid \bot \mid \tau_1=\tau_2 \mid p(\tau_1,\dots,\tau_n)
\mid \neg \phi
\mid \phi \land \psi \mid \phi \lor \psi \mid \phi \rightarrow \psi \\
&\quad \mid \forall x\, \phi \mid \exists x\, \phi
\mid \X\phi \mid \WX\phi \mid \F\phi \mid \G\phi
\mid \phi\,\U\,\psi \mid \phi\,\R\,\psi .
\end{aligned}
$$
Here $x \in \mathcal{X}$, $c \in \mathcal{C}$, $f \in \mathcal{F}$ is an $n$-ary
function symbol, and $P \in \mathcal{P}$ is an $n$-ary predicate symbol.
Terms and atomic formulas are assumed to be well-typed according to
$\delta$, $\delta_{\mathrm{in}}$, and $\delta_{\mathrm{out}}$.
The constants $\top$ and $\bot$ denote the fuzzy truth constants $1$ and $0$,
respectively. Bounded temporal operators, such as $\F^{\le k}\phi$, $\G^{\le k}\phi$, and
$\phi\,\U^{\le k}\,\psi$ for $k \in \mathbb{N}$, are treated as syntactic
abbreviations whose semantics is obtained by restricting the corresponding
temporal aggregation window to the next $k$ time steps.

\subsection{Grounded Semantics of LTRL\textsubscript{f}}\label{sec:semantics}


The semantics of LTRL\textsubscript{f} follows the one in~\citet{DBLP:journals/corr/abs-1203-6278} but extended to finite traces and first-order predicates instead of standard propositional logic. We operate under the \emph{Constant Domain Assumption}, which states that the grounded domain of quantification $\mathbb{D}$ remains fixed at every timestep, and that objects are not created or destroyed over time. The elements of $\mathbb{D}$ correspond to logical entities in the trace and serve as rigid designators across time. Thus, each variable $x \in \mathcal{X}$ ranges over the same finite grounded domain at every time step, with $\mathcal{G}(x)$ denoting a rigid grounded entity. Importantly, this assumption concerns the identity of the objects quantified over, not the particular neural state through which they are observed. While logical entities are rigid designators across the trace, their observable features naturally evolve. Accordingly, the temporal grounding $\mathcal{G}_t(P)(x)$ evaluates predicate $P$ on entity $x$ at time $t$. 
This allows us to separate object quantification from temporal evaluation: temporal operators move only along the trace index, while quantifiers aggregate over the constant grounded domain.
\begin{solution}
\textcolor{blue}{
It is worth noting that under this assumption, same-polarity combinations such as $\forall x\mathcal{G}(\phi)$ and $\mathcal{G}(\forall x\phi)$ can be written interchangeably under the same aggregation scheme, whereas this does not hold for mixed-polarity combinations.}
\end{solution}
In the following, we assume that the grounding of a formula $\phi$ is evaluated at the initial instant, i.e.\ $\mathcal{G}(\phi) := \mathcal{G}_0(\phi)$.

\begin{definition}[Temporal Grounding]
A temporal grounding $\mathcal{G}$ extends the groundings of Real Logic with a finite time domain $\Tmf=[0,l]$, with $l \in \mathbb{N}$, and $l+1$ groundings, $\G_t$, for each $t\in \Tmf$.
Each predicate $p$ 
is grounded as a time-indexed truth-value map
$
\mathcal{G}_t(p): \mathcal{G}(\delta_{\mathrm{in}}(p)) \to [0,1],
$
where $\mathcal{G}(\delta_{\mathrm{in}}(p))$ denotes the Cartesian product of the grounded input domains associated with $p$, and $t\in \Tmf$.
\end{definition}

\paragraph{Grounding of terms and atomic formulas.} 
The grounding of terms is independent from time and is as defined in section~\ref{sub:sec:RL}.
Since variables are grounded by $\mathcal{G}$ to a sequence of elements then the grounding of function  applications to terms containing variables is no more a single element in $\delta_{out}(f)$ but a tensor of dimension $(|\mathcal{G}(\tau_1)|,\ldots,|\mathcal{G}(\tau_n|)$ whose elements are in $\delta_{out}(f)$, capturing in this way the element-wise application of $f$ to the grounding of its parameters:
$\mathcal{G}\big(f(\tau_1,\dots,\tau_n)\big) =
\mathcal{G}(f)\big(\mathcal{G}(\tau_1),\dots,\mathcal{G}(\tau_n)\big).
$
Note that, in case the parameters are just constants then the above grounding returns a single element in $\delta_{out}(f)$.
Similarly, the grounding of an atomic predicate $p$ at time $t \in \Tmf$ is a tensor of dimension $(|\mathcal{G}(\tau_1)|,\ldots,|\mathcal{G}(\tau_n|)$ whose elements are in $[0,1]$ given by
$\mathcal{G}_t\big(p(\tau_1,\dots,\tau_n)\big)
=
\mathcal{G}_t(p)\big(\mathcal{G}(\tau_1),\dots,\mathcal{G}(\tau_n)\big).
$ Just as an example, after grounding a unary temporal predicate over a variable $x$ such that $|\mathcal{G}(x)| = n_1$ we obtain $l+1$ temporal tensors each in $[0,1]^{n_1}$, while a binary temporal predicate, $p(x,y)$, with $|\mathcal{G}(y)| = n_2$ is a tensor $\mathcal{G}_t(p)(\mathcal{G}(x),\mathcal{G}(y))$ in $[0,1]^{n_1 \times n_2}$, for each $t \in \Tmf$. In the following we will also use, e.g., the tensor in  $[0,1]^{n_1 \times n_2 \times (l+1)}$ to compactly represent all the tensors for the grounding of the binary temporal predicate $p(x,y)$ over the $l+1$ time points.

\paragraph{Grounding of equality.}
The equality between two terms $\tau_1$ and $\tau_2$ is grounded 
by a function representing the degree of equivalence of the terms. Such function is based on euclidean distance and defined as follows:
$
\mathcal{G}_t(\tau_1=\tau_2) = \exp\left(-\alpha \times \|\mathcal{G}(\tau_1) - \mathcal{G}(\tau_2)\|_2\right),
$
where $\alpha>0$ is a hyperparameter and the distance is computed element-wise. 
\paragraph{Grounding of formulas.} 
The grounding $\mathcal{G}_t(\phi)$ for a formula $\phi$ without temporal operators is the fuzzy truth value of $\phi$ at time $t$ according to the semantics of first-order fuzzy logic:
\begin{alignat*}{2}
\mathcal{G}_t(\neg \phi)      &= N(\mathcal{G}_t(\phi)), \quad &
\mathcal{G}_t(\phi \land \psi) &= T(\mathcal{G}_t(\phi), \mathcal{G}_t(\psi)), \\[2mm]
\mathcal{G}_t(\phi \lor \psi) &= S(\mathcal{G}_t(\phi), \mathcal{G}_t(\psi)), \quad &
\mathcal{G}_t(\phi \rightarrow \psi)
&= I(\mathcal{G}_t(\phi), \mathcal{G}_t(\psi)), \\[2mm]
\mathcal{G}_t(\forall x_1,\ldots,x_h\,\phi)
&= \underset{\mathcal{G}(x_1,\ldots,x_h)}{\agg_{\forall}} \mathcal{G}_t(\phi), \quad &
\mathcal{G}_t(\exists x_1,\ldots,x_h\,\phi)
&= \underset{\mathcal{G}(x_1,\ldots,x_h)}{\agg_{\exists}} \mathcal{G}_t(\phi)
\end{alignat*}
where negation, conjunction, disjunction and implication are associated, respectively, with a fuzzy negation ($N$), a t-norm ($T$), a t-conorm ($S$) and a fuzzy implication ($I$). The semantics of the quantifiers is defined with a continuous aggregation operator $\agg: \bigcup_{n\in \mathbb{N}}[0,1]^n \rightarrow [0,1]$ as done for LTN. This operator aggregates over the grounding of the tuple of variables $(x_1,\ldots,x_h)$ contained in $\phi$, i.e., $\mathcal{G}(x_1,\ldots,x_h) = \bigtimes_{i=1}^h \mathcal{G}(x_i)$. We adopt the same operators as in LTN, the generalized $p-$mean $A^{\mathrm{M}}_p$ for $\exists$ and the $p$-mean error $A^{\mathrm{ME}}_p$ for $\forall$:
\begin{equation}
A^{\mathrm{M}}_p(u_1,\dots,u_m)=\!\left(\tfrac{1}{m}\textstyle\sum_{i=1}^m u_i^p\right)^{\!1/p}\!,\quad A^{\mathrm{ME}}_p(u_1,\dots,u_m)=1-\!\left(\tfrac{1}{m}\textstyle\sum_{i=1}^m(1{-}u_i)^p\right)^{\!1/p}\!
\label{eq:aggr}
\end{equation}
where $p \geq 1$ is a hyperparameter dictating the strictness of the approximation.

We extend the above-defined groundings to support temporal operators where the time index $t$ now ranges in $[0,l]$:
\begin{alignat*}{2}
\mathcal{G}_t(\X\phi) &= \begin{cases} \mathcal{G}_{t+1}(\phi) & t<l,\\ 
0 & t=l,\end{cases}\label{eq:semanticsLTRL\textsubscript{f}}
&
\mathcal{G}_t(\WX\phi) &= \begin{cases} \mathcal{G}_{t+1}(\phi) & t<l,\\ 
1 & t=l,\end{cases}
\\[1em]
\mathcal{G}_t(\F\phi) &= \bigoplus_{t'\in[t,l]} \mathcal{G}_{t'}(\phi),
&
\mathcal{G}_t(\G\phi) &= \bigotimes_{t'\in[t,l]} \mathcal{G}_{t'}(\phi),
\\[1em]
\mathcal{G}_t(\phi\,\U\,\psi) &= \bigoplus_{t'\in[t,l]} T\!\left(\mathcal{G}_{t'}(\psi),\;\bigotimes_{v\in[t,t'-1]}\mathcal{G}_v(\phi)\right), \qquad
&
\mathcal{G}_t(\phi\,\R\,\psi) &= N\!\left(\mathcal{G}_t(\neg\phi\,\U\,\neg\psi)\right),
\end{alignat*}
where $\bigoplus$ and $\bigotimes$ denote the fuzzy supremum and infimum operators (with the standard conventions $\bigotimes\limits_{\varnothing}=1$ and $\bigoplus\limits_{\varnothing}=0$) induced by the chosen t-conorm $S$ and t-norm $T$, respectively, as in \citet{DBLP:journals/corr/abs-1203-6278}.
%
%
%
%
%
The Release operator is defined by duality from Until; in this way, it is consistent with the adopted reflexive interpretation of until, where the witnessing time point may be the current one ($t' = t$). In particular, the case $t'=t$ makes $\phi\,\U\,\psi$ immediately satisfied to degree $\mathcal{G}_t(\psi)$, since $\bigotimes\limits_{\varnothing}=1$. Bounded variants can be defined by restricting the aggregation to an interval $[t,\min(t{+}k,l)]$. The use of fuzzy connectives, aggregations, infimum and supremum operators requires additional proofs of the logical equivalences at the end of Section~\ref{sec:foltlf}. Proofs exist in~\cite{DBLP:journals/procsci/DonadelloFIMM25} for propositional fuzzy LTL\textsubscript{f} with the Gödel t-norm.


\paragraph{Differentiability of LTRL\textsubscript{f}}
The operator $\bigoplus$ is implemented as the generalized p-mean $A^{\mathrm{M}}_p$ and $\bigotimes$ as the p-mean error $A^{\mathrm{ME}}_p$, the same aggregators used for existential and universal quantification in LTN, because $A^{\mathrm{M}}_p$ and $A^{\mathrm{ME}}_p$ provide smooth and differentiable approximations of the max and min operators, while converging to them as $p \to \infty$. This preserves the intended fuzzy temporal semantics of operators such as $\F$ and $\G$, while enabling efficient gradient-based optimization. Moreover, using the same aggregators already adopted in LTN for the quantifications, ensures semantic consistency between first-order and temporal reasoning: the temporal operators $\F$ and $\G$ can thus be viewed as temporal counterparts of $\exists$ and $\forall$, respectively, differing only in the domain of aggregation (time points rather than individuals). 

Regarding the choices of t-norms, t-conorms, negation and fuzzy implications, we  adopted the \emph{stable product configuration} (Appendix~\ref{apd:operators}) that is well suited for gradient-descent optimization. Appendix~\ref{apd:example} shows an example of the temporal grounding computation according to the stable product configuration. When the grounding of predicates and functions is implemented through a differentiable function over parameters $\theta$ (e.g., neural networks) the semantics of LTRL\textsubscript{f} is differentiable, see proofs in Appendix~\ref{apd:diff}.

\subsection{Learning in FOT-LTN}
FOT-LTN learns the groundings of predicates/functions (with neural networks) by maximizing the satisfaction of supervised (time-dependent) examples $s_t^{(i)} \in  \mathcal{S}$ and of temporal logical axioms in a knowledge base ($\mathcal{BK}$). Given a set of supervised examples and a temporal knowledge base, the truth values of formulas are computed according to the differentiable semantics of LTRL\textsubscript{f}, and the neural network components are trained to find the network parameters $\theta^*$ that maximize the aggregate satisfaction of both data and logical constraints: $$\theta^* = \argmax_\theta \agg ( \{\mathcal{G}_t(s_t^{(i)})\}_{s_t^{(i)} \in \mathcal{S}} \cup \{\mathcal{G}_0(\phi)\}_{\phi \in \mathcal{BK}}).$$

\section{Experiments}\label{sec:experiments}
We evaluate FOT-LTN on a controlled experiment regarding temporal knowledge graph (TKG) completion on a synthetic car-pedestrian scenario. The goal is to train FOT-LTN on a subset of temporal facts and query the remaining ones. Our research questions regard the higher predictive performance of FOT-LTN w.r.t. purely data-driven models (\textbf{RQ1}) and its capacity of generating predictions that are compliant with background knowledge (\textbf{RQ2}). The source code of the experiments is available as supplementary material. 
\begin{solution}
\textcolor{blue}{
More details regarding the implementation are in Appendix~\ref{apd:implementation}.}
\end{solution}

\subsection{Dataset and Background Knowledge}
We generate two synthetic datasets (CarPed35K, CarPed180K) that model interactions between cars and pedestrian during time instants. The predicates are $\mathsf{Car}$, $\mathsf{Ped}$, $\mathsf{Run}$ and $\mathsf{Stop}$ that describe car states, while $\mathsf{Crossing}$ and $\mathsf{OnSidewalk}$ describe pedestrian states. These are learnable predicates, i.e., their grounding is learned by a data-driven model. The predicate $\mathsf{Close}(x,y)$ relates pairs of nearby entities and is manually defined. This is a reasonable choice as a close relationship can be computed with data coming from cars’ sensors. Our aim is to test FOT-LTN in presence of both learnable and rule-based predicates. The datasets are generated stochastically, using 50 instances of cars, 50 pedestrians, $T = 100$ for CarPed35K and 100 cars, 100 pedestrians and $T = 150$ for CarPed180K. For each pedestrian, we sample uniformly either one or two crossing intervals, whose start is sampled from $\{0,\dots,T-1\}$ and duration from $\{1,\dots,\lfloor T/5\rfloor -1 \}$, clipped at the end of the trace. The $\mathsf{Close}$ predicate facts are sampled from a Bernoulli distribution with probability 0.05 and then made symmetric. Each car instance is initially running. Then, at each time $t < T-1$, if the car is close to at least one crossing pedestrian, it is set to stop at $t+1$. Cars can only run or stop, pedestrian can only cross or stay on the sidewalk.
In total, CarPed35K and CarPed180K have 35,090 and 179,332 positive facts, respectively, see the statistics in Table~\ref{tab:carped_predicate_statistics}. All the other (not instantiated) facts are considered negatives.
\begin{table}[t]
\centering
\caption{Statistics of datasets with the number of positive/total facts and their percentage.}
\label{tab:carped_predicate_statistics}
\setlength{\tabcolsep}{4pt}
\begin{tabular}{lcc}
\toprule
\textbf{Predicate}
& \textbf{CarPed35K}
& \textbf{CarPed180K} \\
\midrule
Car        & $50/100$ \; {\scriptsize $(50.0\%)$}       & $100/200$ \; {\scriptsize $(50.0\%)$} \\
Ped        & $50/100$ \; {\scriptsize $(50.0\%)$}       & $100/200$ \; {\scriptsize $(50.0\%)$} \\
Run        & $3550/5000$ \; {\scriptsize $(71.0\%)$}    & $8085/15000$ \; {\scriptsize $(53.9\%)$} \\
Stop       & $1450/5000$ \; {\scriptsize $(29.0\%)$}    & $6915/15000$ \; {\scriptsize $(46.1\%)$} \\
Crossing   & $687/5000$ \; {\scriptsize $(13.7\%)$}     & $1923/15000$ \; {\scriptsize $(12.8\%)$} \\
OnSidewalk & $4313/5000$ \; {\scriptsize $(86.3\%)$}    & $13077/15000$ \; {\scriptsize $(87.2\%)$} \\
Close      & $24990/10^6$ \; {\scriptsize $(2.5\%)$}    & $149132/(6{\times}10^6)$ \; {\scriptsize $(2.5\%)$} \\
\bottomrule
\end{tabular}
\end{table}

The background knowledge $\mathcal{BK}$ (Table~\ref{tab:axioms}) has been manually defined. Axioms A1-A3 encode the cars behavior according to the crossing behavior (and closeness) of pedestrians. Axiom A4 makes all cars running at the first time step. Axiom A5 encodes the symmetry of the $\mathsf{Close}$ predicate. Axioms A6-A8 encode the actions allowed for cars: only running or stopping in mutual exclusion. A9 states that cars and pedestrian are disjoint concepts. A10-A12 encode the actions allowed for pedestrians: only crossing or staying on the sidewalk in mutual exclusion. A13 states that a crossing action will end on the sidewalk.
\begin{table}[t]
\caption{The background knowledge $\mathcal{BK}$.}
\centering
\footnotesize
\begin{tabular}{@{}cl@{}}
\toprule
\textbf{ID} & \textbf{Axiom} \\
\midrule
A1  & $\forall x\,\forall y\;\G(\mathsf{Car}(x)\land\mathsf{Ped}(y)\land\mathsf{Crossing}(y)\land\mathsf{Close}(x,y)\to\WX\,\mathsf{Stop}(x))$ \\
A2  & $\forall x\,\forall y\;\G(\mathsf{Car}(x)\land\mathsf{Stop}(x)\land\mathsf{Ped}(y)\land\mathsf{Crossing}(y)\land\mathsf{Close}(x,y)\to\F\,\mathsf{Run}(x))$ \\
A3 & $\forall x\;\G(\mathsf{Car}(x)\land\neg\exists y\,(\mathsf{Ped}(y)\land\mathsf{Crossing}(y)\land\mathsf{Close}(x,y))\to\WX\,\mathsf{Run}(x))$ \\
A4 & $\forall x\;(\mathsf{Car}(x)\to\mathsf{Run}(x))$ \\
A5  & $\forall x\,\forall y\;\G(\mathsf{Close}(x,y)\to\mathsf{Close}(y,x))$ \\
A6  & $\forall x\;\G(\mathsf{Car}(x)\to\mathsf{Run}(x)\lor\mathsf{Stop}(x))$ \\
A7  & $\forall x\;\G(\mathsf{Run}(x)\to\neg\mathsf{Stop}(x))$ \\
A8 & $\forall x\;\G(\mathsf{Car}(x)\to\neg\mathsf{Crossing}(x)\land\neg\mathsf{OnSidewalk}(x))$ \\
A9  & $\forall x\;\G(\mathsf{Ped}(x)\to\neg\mathsf{Car}(x))$ \\
A10  & $\forall x\;\G(\mathsf{Crossing}(x)\to\neg\mathsf{OnSidewalk}(x))$ \\
A11 & $\forall x\;\G(\mathsf{Ped}(x)\to\mathsf{Crossing}(x)\lor\mathsf{OnSidewalk}(x))$ \\
A12 & $\forall x\;\G(\mathsf{Ped}(x)\to\neg\mathsf{Run}(x))$ \\
A13 & $\forall x\;\G(\mathsf{Crossing}(x)\to\F\,\mathsf{OnSidewalk}(x))$ \\
\bottomrule
\end{tabular}
\label{tab:axioms}
\end{table}

\subsection{Experimental Design}
\begin{solution}
\todo[inline]{This should be inside FOT-LTN + BK}    
{\color{blue}
As an implementation detail, we set the embedding size to 32 for all FOT-LTN models, although this remains a tunable hyperparameter in the released implementation.
    
The model learns an embedding tensor $G \in \mathbb{R}^{|\mathcal{D}|\times T \times 32}$, with one 32-dimensional embedding for each entity and time instant. The static type predicates $\mathsf{Car}$ and $\mathsf{Ped}$ are evaluated from the initial embedding $G_{x,0}$ and then broadcast over time. The temporal predicates $\mathsf{Run}$, $\mathsf{Stop}$, $\mathsf{Crossing}$, and $\mathsf{OnSidewalk}$ are evaluated independently at each pair $(x,t)$ from $G_{x,t}$.

To encode the datasets as a TKG, we map each entity instance to a node, typed with $\mathsf{Car}$ or $\mathsf{Ped}$. The binary predicate $\mathsf{Close}$ is encoded as temporal edges between entities, while the unary predicates ($\mathsf{Run}$, $\mathsf{Stop}$, $\mathsf{Crossing}$ and $\mathsf{OnSidewalk}$), are not represented as graph edges. Instead, they are treated as node-level prediction targets and are predicted from the evolving temporal node embeddings using dedicated MLP heads.

\[
\mathrm{KB-SAT}
=
\left(
\prod_{i=1}^{|\mathcal{BK}|}
\mathcal{G}_0(\phi_i)
\right)^{\frac{1}{|\mathcal{BK}|}}
=
\exp\left(
\frac{1}{|\mathcal{BK}|}
\sum_{i=1}^{|\mathcal{BK}|}
\log\bigl(\mathcal{G}_0(\phi_i)\bigr)
\right).
\]}
\end{solution}

We evaluate FOT-LTN in a TKG completion task on the CarPed35K and CarPed180K datasets (with both positive and negative facts) with different levels of data availability (from 10\% to 80\%) for training and the rest for test. To make the datasets more realistic, we add some random noise by flipping the truth value of the 10\% of total (positive and negative) training facts. This noise injection affects only the training set whereas we keep the test set clean for a more clear evaluation. To encode the datasets as a TKG, we map each entity instance to a node, typed with $\mathsf{Car}$ or $\mathsf{Ped}$. The binary predicate $\mathsf{Close}$ is encoded as temporal edges between entities, while the unary predicates ($\mathsf{Run}$, $\mathsf{Stop}$, $\mathsf{Crossing}$ and $\mathsf{OnSidewalk}$), are represented as self-temporal edges. We repeat the evaluation 10 times to obtain statistically robust results. We compare three different methods:
\begin{compactitem}[1.]
    \item \textbf{FOT-LTN+$\mathcal{BK}$}: a FOT-LTN predictive model that considers both training examples and the axioms in $\mathcal{BK}$. The loss function $\mathcal{L} = A^{\mathrm{ME}}_p ( \mathit{sat}_{\mathrm{sup}}, \mathit{sat}_{\mathcal{BK}})$ aggregates the satisfiability of the training examples with the one of $\mathcal{BK}$. Here, $\mathit{sat}_{\mathrm{sup}} = A^{\mathrm{ME}}_p(ce_0^{(1)},\dots,ce_T^{(n)})$ where $ce_t^{(i)}$ is the cross entropy between the $i$-th supervised fact $P_t^{(i)}$ (binary or unary) at time $t$ and the corresponding grounding $\mathcal{G}_t(P_t^{(i)})$, i.e., $ce_t^{(i)} = y_t^{(i)}\mathcal{G}_t(P_t^{(i)}) + (1 - y_t^{(i)})(1-\mathcal{G}_t(P_t^{(i)})$, where $y_t^{(i)} \in \{0,1\}$ is the ground-truth label of $P_t^{(i)}$. The $\mathcal{BK}$ satisfiability is $\mathit{sat}_{\mathcal{BK}} = A^{\mathrm{ME}}_p(\mathcal{G}_0(\phi_1),\dots,\mathcal{G}_0(\phi_{|\mathcal{BK}|}))$. We set $p=4$ to highly penalize when one of the two satisfiabilities is low. Each learnable predicate is implemented as an MLP with two hidden layers of 16 ELU units and a sigmoid output.
    \item \textbf{FOT-LTN}: This is an ablation study of the first method, with the same MLPs but without $\mathcal{BK}$. Hence, the loss function $\mathcal{L} = 1 - \mathit{sat}_{\mathrm{sup}}$ considers only the supervision loss.
    \item \textbf{HTGNN}: A heterogeneous temporal graph neural network~\citep{DBLP:journals/corr/abs-2110-13889} developed for predictions tasks on temporal graphs. HTGNN is much more sophisticated than the MLPs used above. Indeed, HTGNN has 1.81x and 1.39x more parameters than FOT-LTN (with and without $\mathcal{BK}$) for CarPed35K and CarPed180K, respectively.
\end{compactitem}
For each method, we use Adam optimizer with a learning rate of $0.005$. The number of epochs is 500 and early stopping is used with patience 30. We use the stable product configuration for the choice of the fuzzy logical connectives and $p=4$ in Equation~\ref{eq:aggr}. 

The performance metrics (both the higher the better) reflect the research questions:
\begin{compactitem}[$\bullet$]
    \item \textbf{PR-AUC}: Area under the precision-recall curve averaged over all the learnable predicates. This measures the link prediction performance and answers to $\textbf{RQ1}$.
    \item \textbf{KB-SAT}: The satisfaction of the $\mathcal{BK}$ axioms according to the learned predicates for answering to $\textbf{RQ2}$. We use the geometric mean (in its logarithmic form) as it offers an interpretable measure of the overall knowledge-base consistency that penalizes low-satisfaction axioms: $\text{KB-SAT} =\exp\left(\frac{1}{|\mathcal{BK}|} \sum_{\phi_i \in \mathcal{BK}}\log\bigl(\mathcal{G}_0(\phi_i)\bigr)\right)$.
\end{compactitem}

\subsection{Results}
Table~\ref{tab:results_exp2} shows the numeric results on the test set according to each level of training data availability (plots are in Appendix~\ref{apd:res_plot}).
\begin{table}[t]
\caption{PR-AUC and KB-Satisfaction. Each entry is \emph{mean $\pm$ std} over 10 independent runs.}
\label{tab:results_exp2}
\centering
\scriptsize
\setlength{\tabcolsep}{2pt}
\renewcommand{\arraystretch}{0.88}
\resizebox{\linewidth}{!}{%
\begin{tabular}{llcccccccc}
\toprule
\textbf{Dataset} & \textbf{Model} & \textbf{10\%} & \textbf{20\%} & \textbf{30\%} & \textbf{40\%} & \textbf{50\%} & \textbf{60\%} & \textbf{70\%} & \textbf{80\%} \\
\midrule
\multicolumn{10}{c}{\textit{PR-AUC}} \\
\midrule
\multirow{3}{*}{CarPed35K}
 & FOT-LTN+$\mathcal{BK}$     & \textbf{.727$\pm$.010} & \textbf{.726$\pm$.012} & \textbf{.736$\pm$.010} & \textbf{.754$\pm$.011} & \textbf{.772$\pm$.009} & \textbf{.795$\pm$.010} & \textbf{.819$\pm$.010} & \textbf{.838$\pm$.011} \\
 & FOT-LTN & .597$\pm$.020 & .617$\pm$.041 & .629$\pm$.030 & .631$\pm$.039 & .670$\pm$.029 & .688$\pm$.024 & .710$\pm$.026 & .742$\pm$.036 \\
 & HTGNN          & .616$\pm$.025 & .660$\pm$.016 & .702$\pm$.027 & .730$\pm$.025 & .765$\pm$.026 & .793$\pm$.031 & .807$\pm$.034 & .828$\pm$.022 \\
\midrule
\multirow{3}{*}{CarPed180K}
 & FOT-LTN+$\mathcal{BK}$     & \textbf{.719$\pm$.007} & \textbf{.713$\pm$.015} & \textbf{.715$\pm$.009} & \textbf{.729$\pm$.008} & .744$\pm$.010 & .766$\pm$.009 & .797$\pm$.007 & .820$\pm$.012 \\
 & FOT-LTN & .574$\pm$.026 & .607$\pm$.035 & .631$\pm$.029 & .648$\pm$.023 & .680$\pm$.028 & .708$\pm$.024 & .746$\pm$.029 & .795$\pm$.029 \\
 & HTGNN          & .603$\pm$.021 & .634$\pm$.039 & .679$\pm$.033 & .723$\pm$.032 & \textbf{.758$\pm$.036} & \textbf{.786$\pm$.023} & \textbf{.802$\pm$.020} & \textbf{.831$\pm$.016} \\
\midrule
\multicolumn{10}{c}{\textit{KB-Satisfaction}} \\
\midrule
\multirow{3}{*}{CarPed35K}
 & FOT-LTN+$\mathcal{BK}$     & \textbf{.943$\pm$.010} & \textbf{.916$\pm$.009} & \textbf{.903$\pm$.009} & \textbf{.891$\pm$.006} & \textbf{.884$\pm$.004} & \textbf{.878$\pm$.003} & \textbf{.873$\pm$.004} & \textbf{.868$\pm$.003} \\
 & FOT-LTN & .572$\pm$.038 & .524$\pm$.036 & .534$\pm$.021 & .515$\pm$.020 & .521$\pm$.024 & .510$\pm$.012 & .517$\pm$.017 & .507$\pm$.013 \\
 & HTGNN          & .568$\pm$.037 & .565$\pm$.025 & .558$\pm$.026 & .558$\pm$.019 & .551$\pm$.028 & .547$\pm$.025 & .535$\pm$.028 & .524$\pm$.019 \\
\midrule
\multirow{3}{*}{CarPed180K}
 & FOT-LTN+$\mathcal{BK}$ & \textbf{.939$\pm$.008} & \textbf{.918$\pm$.007} & \textbf{.904$\pm$.004} & \textbf{.892$\pm$.004} & \textbf{.888$\pm$.003} & \textbf{.882$\pm$.003} & \textbf{.878$\pm$.002} & \textbf{.873$\pm$.002} \\
 & FOT-LTN & .549$\pm$.035 & .536$\pm$.034 & .519$\pm$.018 & .519$\pm$.024 & .513$\pm$.017 & .510$\pm$.014 & .517$\pm$.012 & .505$\pm$.012 \\
 & HTGNN          & .572$\pm$.053 & .576$\pm$.041 & .559$\pm$.031 & .561$\pm$.028 & .558$\pm$.036 & .548$\pm$.038 & .542$\pm$.034 & .529$\pm$.028 \\
\bottomrule
\end{tabular}%
}
\end{table}
FOT-LTN+$\mathcal{BK}$ presents higher PR-AUC than both FOT-LTN and HTGNN when training data availability is scarce, and comparable performance with larger percentage of data. We stress the fact that FOT-LTN+$\mathcal{BK}$ is a general framework, not tailored for TKG completion as HTGNN and with less parameters. This indicates that including temporal axioms was particularly effective, mostly in low-data regimes. As supervision increases, the PR-AUC gap between axiom-guided and purely supervised models is expected to narrow, since the target predicates become increasingly observed. For KB-SAT, FOT-LTN+$\mathcal{BK}$ dominates the other approaches, displaying constant high values across all percentages. On the other hand, FOT-LTN and HTGNN always achieve a very low KB-SAT even when the PR-AUC is high. 
This reveals good overall prediction performance of the non-Nesy methods but non-compliant for some time instants whereas FOT-LTN+$\mathcal{BK}$ ensures largely compliant predictions. This is crucial for some critical applications, such as a car-pedestrian scenario, where failing a prediction even in a single instant could lead to catastrophic results. We also performed some statistical tests to understand at which training data availability FOT-LTN + $\mathcal{BK}$ does not outperform HTGNN anymore on the PR-AUC. FOT-LTN + $\mathcal{BK}$ significantly outperforms HTGNN till $30\%$ of data supervision. From $40\%$ onward, the difference is not statistically significant (see Appendix~\ref{apd:statistical_test}). Therefore, we can positively answer to both our research questions.

\begin{solution}
{\color{blue}Appendix~\ref{apd:res_plot} contains some plots that show the results of Table~\ref{tab:results_exp2}. In addition, the embeddings for cars and pedestrians according to the two principal components are shown. The two sets of individuals are linearly separable indicating that FOT-LTN has learned useful embeddings according the objects behaviors.}

{\color{blue} Since implication-based axioms with sparse antecedents can be satisfied vacuously \citep{DBLP:journals/corr/abs-2002-06100}, KB-Satisfaction should be interpreted as a measure of logical consistency rather than as a direct substitute for predictive performance. Furthermore, reasoning shortcuts \citep{DBLP:journals/corr/abs-2510-14538} should also be taken into consideration, as vacuous satisfaction prevents the model to learn rules properly.}
\end{solution}

\section{Related Work}\label{sec:related}
One prominent Neuro-Symbolic (NeSy) approach is the encoding of logical knowledge into a differentiable form. Logic Tensor Networks (LTNs)~\citep{serafini_garcez_2020} is based on the differentiable Real Logic. Each prediction is evaluated against a knowledge base and a semantic loss term is hereby computed. Another approach is to encode the knowledge directly into the neural architecture, see, for instance, Logical Neural Networks~\citep{DBLP:journals/corr/abs-2006-13155}. Other approaches include probabilistic reasoning, for example the work in~\citep{DBLP:conf/nips/ManhaeveDKDR18} extends probabilistic logic programming with neural predicates. More works can be found dedicated in surveys~\citep{DBLP:journals/ai/MarraDMR24}. While achieving solid performances, the above-mentioned systems focus on static, non-temporal knowledge.

Given the importance of time-related constraints, new methods are investigating the topic. An early temporal NeSy event recognition model~\citep{DBLP:conf/time/ApricenoPS21} focuses on structured event recognition under shallow supervision. The work in~\citep{DBLP:journals/corr/abs-2303-17892} introduces Interval Logic Tensor Networks that extends Real Logic with interval-based reasoning, via fuzzy intervals and relations between events. Another temporal aspect regards the flow of events through a linear-time formalism expressed by LTL on finite traces (LTL\textsubscript{f}). The work in ~\citep{DBLP:journals/corr/abs-2405-06670} explores differentiable LTL\textsubscript{f} constraints through a tailored smooth (real-valued) semantics for video activity recognition. Differentiable LTL\textsubscript{f} constraints are leveraged also in~\citep{DBLP:conf/nesy/AndreoniBDGMR25} with a standard LTL\textsubscript{f} fuzzy semantics~\citep{DBLP:journals/procsci/DonadelloFIMM25} for image-sequence classification. Other LTL\textsubscript{f}-based losses are encoded through differentiable automata for next activity prediction in Predictive Process Monitoring (PPM)~\citep{nesyppm} and sequence-image classification~\citep{DBLP:conf/ijcai/ManginasPR25}. A different approach encodes LTL\textsubscript{f} constraints into an automaton so as to guide a beam search algorithm at inference time to predict sequences of symbols (traces) that are compliant with the constraints. This has been tested in PPM with a probabilistic fragment of LTL\textsubscript{f}~\citep{10.1145/3810944} and in constrained text generation for Large Language Models~\citep{collura2025absenforcingconstraintsatisfaction}. These methods are limited to propositional LTL\textsubscript{f} and do not consider objects and their relations as done by FOT-LTN. Moreover, the direct (differentiable) semantics evaluation performed by FOT-LTN avoids the construction of automata that requires exponential time. The work in ~\citep{DBLP:conf/ijcai/Lorello0M25} consists of a multi-stage NeSy architecture that combines perception, relational and temporal reasoning within a unified sequence-classification framework. The relational component relies on Datalog-based reasoning, while temporal reasoning is performed over propositional temporal specifications. Unlike FOT-LTN, this temporal language does not support explicit quantification over individuals across time obtaining a less expressive NeSy framework. The proposed architecture is primarily designed for sequence-classification and does not directly support more general temporal reasoning problems, such as TKG completion.


\section{Conclusion}\label{sec:conclusion}
We introduced First-Order Temporal Logic Tensor Networks (FOT-LTN) a new NeSy framework that jointly learns and reasons with objects whose properties, and relations, vary according to a linear time domain over finite traces. Up to our knowledge, this is the first NeSy framework with such an expressivity. We provide a first validation on a task of temporal knowledge graph completion obtaining higher classification performance w.r.t. a dedicated deep learning method in presence of scarce supervision and predictions much more compliant with the input knowledge for all supervision levels. As future work, more experiments on real datasets involving perception (e.g., temporal visual-question answering or semantic video interpretation) will be performed as well as a study on the reasoning capabilities and time complexity of FOT-LTN. In addition, the expressivity will be increased by adding past operators in the logical language.

\acks{Acknowledgements omitted for anonymous review.}

\bibliography{main}

@BOOK{Di_Francescomarino2026-uy,
  title     = {Predictive Process Monitoring},
  author    = {Chiara {D}i {F}rancescomarino and Ivan Donadello and Fabrizio Maria Maggi},
  isbn         = {978-3-032-17277-8},
    doi          = {10.1007/978-3-032-17278-5},
  publisher = {Springer Nature Switzerland},
  year      =  {2026}
}

@incollection{SerafiniGBDSB21,
  author       = {Luciano Serafini and
                  Artur S. d'Avila Garcez and
                  Samy Badreddine and
                  Ivan Donadello and
                  Michael Spranger and
                  Federico Bianchi},
  editor       = {Pascal Hitzler and
                  Md. Kamruzzaman Sarker},
  title        = {Logic {T}ensor {N}etworks: Theory and {A}pplications},
  booktitle    = {Neuro-Symbolic Artificial Intelligence: The State of the Art},
  series       = {Frontiers in Artificial Intelligence and Applications},
  volume       = {342},
  pages        = {370--394},
  publisher    = {{IOS} Press},
  year         = {2021},
  doi          = {10.3233/FAIA210498},
  bibsource    = {dblp computer science bibliography, https://dblp.org}
}

@inproceedings{DBLP:conf/nesy/SerafiniG16,
  author       = {Luciano Serafini and
                  Artur S. d'Avila Garcez},
  title        = {Logic Tensor Networks: Deep Learning and Logical Reasoning from Data
                  and Knowledge},
  booktitle    = {NeSy@HLAI},
  series       = {{CEUR} Workshop Proceedings},
  publisher    = {CEUR-WS.org},
  year         = {2016}
}

@article{serafini_garcez_2020,
  title     = {Logic Tensor Networks},
  volume    = {303},
  ISSN      = {0004-3702},
  url       = {http://dx.doi.org/10.1016/j.artint.2021.103649},
  DOI       = {10.1016/j.artint.2021.103649},
  journal   = {Artificial Intelligence},
  publisher = {Elsevier BV},
  author    = {Badreddine, Samy and d'Avila Garcez, Artur and Serafini, Luciano and Spranger, Michael},
  year      = {2022},
  month     = feb,
  pages     = {103649}
}

@article{DBLP:journals/corr/abs-1203-6278,
  author       = {Achille Frigeri and
                  Liliana Pasquale and
                  Paola Spoletini},
  title        = {Fuzzy Time in {LTL}},
  journal      = {CoRR},
  volume       = {abs/1203.6278},
  year         = {2012}
}

@inproceedings{DBLP:conf/ijcai/GiacomoV13,
  author       = {Giuseppe De Giacomo and
                  Moshe Y. Vardi},
  title        = {Linear Temporal Logic and Linear Dynamic Logic on Finite Traces},
  booktitle    = {{IJCAI}},
  pages        = {854--860},
  publisher    = {{IJCAI/AAAI}},
  year         = {2013}
}

@article{nesyppm,
author = {Axel Mezini and Elena Umili and Ivan Donadello and Fabrizio Maria Maggi and Matteo Mancanelli and Fabio Patrizi},
title = {Neuro-Symbolic Predictive Process Monitoring},
journal = {Information Systems},
volume = {141},
pages = {102762},
year = {2026},
issn = {0306-4379},
doi = {https://doi.org/10.1016/j.is.2026.102762},
url = {https://www.sciencedirect.com/science/article/pii/S0306437926000761}
}

@article{10.1145/3810944,
author = {Alman, Anti and Comuzzi, Marco and Di Francescomarino, Chiara and Donadello, Ivan and Maggi, Fabrizio Maria and Oukharijane, jamila},
title = {Definitely Maybe: Neuro-Symbolic Predictive Process Monitoring with Probabilistic Declarative Knowledge},
year = {2026},
publisher = {Association for Computing Machinery},
address = {New York, NY, USA},
issn = {2157-6904},
url = {https://doi.org/10.1145/3810944},
doi = {10.1145/3810944},
note = {Just Accepted},
journal = {ACM Trans. Intell. Syst. Technol.},
month = apr
}

@article{DBLP:journals/ai/KriekenAH22,
  author       = {Emile van Krieken and
                  Erman Acar and
                  Frank van Harmelen},
  title        = {Analyzing Differentiable Fuzzy Logic Operators},
  journal      = {Artif. Intell.},
  volume       = {302},
  pages        = {103602},
  year         = {2022},
  url          = {https://doi.org/10.1016/j.artint.2021.103602},
  doi          = {10.1016/J.ARTINT.2021.103602},
  bibsource    = {dblp computer science bibliography, https://dblp.org}
}

@article{DBLP:journals/corr/abs-2405-06670,
  author       = {Danyang Li and
                  Mingyu Cai and
                  Cristian{-}Ioan Vasile and
                  Roberto Tron},
  title        = {TLINet: Differentiable Neural Network Temporal Logic Inference},
  journal      = {CoRR},
  volume       = {abs/2405.06670},
  year         = {2024}
}

@inproceedings{DBLP:conf/nesy/AndreoniBDGMR25,
  author       = {Riccardo Andreoni and
                  Andrei Buliga and
                  Alessandro Daniele and
                  Chiara Ghidini and
                  Marco Montali and
                  Massimiliano Ronzani},
  title        = {{T-ILR:} a Neurosymbolic Integration for LTLf},
  booktitle    = {NeSy},
  series       = {Proceedings of Machine Learning Research},
  pages        = {252--265},
  publisher    = {{PMLR}},
  year         = {2025}
}

@inproceedings{DBLP:conf/ijcai/ManginasPR25,
  author       = {Nikolaos Manginas and
                  George Paliouras and
                  Luc De Raedt},
  title        = {NeSyA: Neurosymbolic Automata},
  booktitle    = {{IJCAI}},
  pages        = {5950--5958},
  publisher    = {ijcai.org},
  year         = {2025}
}

@article{DBLP:journals/procsci/DonadelloFIMM25,
  author       = {Ivan Donadello and
                  Paolo Felli and
                  Craig Innes and
                  Fabrizio Maria Maggi and
                  Marco Montali},
  title        = {{LTL}-based conformance checking of fuzzy event logs},
  journal      = {Process Sci.},
  volume       = {2},
  number       = {1},
  year         = {2025}
}

@article{DBLP:journals/corr/abs-2510-14538,
  author       = {Emanuele Marconato and
                  Samuele Bortolotti and
                  Emile van Krieken and
                  Paolo Morettin and
                  Elena Umili and
                  Antonio Vergari and
                  Efthymia Tsamoura and
                  Andrea Passerini and
                  Stefano Teso},
  title        = {Symbol Grounding in Neuro-Symbolic {AI:} {A} Gentle Introduction to
                  Reasoning Shortcuts},
  journal      = {CoRR},
  volume       = {abs/2510.14538},
  year         = {2025}
}

@article{DBLP:journals/corr/abs-2110-13889,
  author       = {Yujie Fan and
                  Mingxuan Ju and
                  Chuxu Zhang and
                  Liang Zhao and
                  Yanfang Ye},
  title        = {Heterogeneous Temporal Graph Neural Network},
  journal      = {CoRR},
  volume       = {abs/2110.13889},
  year         = {2021}
}

@article{DBLP:journals/corr/abs-2002-06100,
  author       = {Emile van Krieken and
                  Erman Acar and
                  Frank van Harmelen},
  title        = {Analyzing Differentiable Fuzzy Logic Operators},
  journal      = {CoRR},
  volume       = {abs/2002.06100},
  year         = {2020},
  url          = {https://arxiv.org/abs/2002.06100},
  eprinttype   = {arXiv},
  eprint       = {2002.06100},
  bibsource    = {dblp computer science bibliography, https://dblp.org}
}

@article{DBLP:journals/corr/abs-2006-13155,
  author       = {Ryan Riegel and
                  Alexander G. Gray and
                  Francois P. S. Luus and
                  Naweed Khan and
                  Ndivhuwo Makondo and
                  Ismail Yunus Akhalwaya and
                  Haifeng Qian and
                  Ronald Fagin and
                  Francisco Barahona and
                  Udit Sharma and
                  Shajith Ikbal and
                  Hima Karanam and
                  Sumit Neelam and
                  Ankita Likhyani and
                  Santosh K. Srivastava},
  title        = {Logical Neural Networks},
  journal      = {CoRR},
  volume       = {abs/2006.13155},
  year         = {2020},
  url          = {https://arxiv.org/abs/2006.13155},
  eprinttype   = {arXiv},
  eprint       = {2006.13155},
  bibsource    = {dblp computer science bibliography, https://dblp.org}
}

@inproceedings{DBLP:conf/ijcai/ArtaleMO19,
  author       = {Alessandro Artale and
                  Andrea Mazzullo and
                  Ana Ozaki},
  title        = {Do You Need Infinite Time?},
  booktitle    = {{IJCAI}},
  pages        = {1516--1522},
  publisher    = {ijcai.org},
  year         = {2019}
}

@article{DBLP:journals/tocl/ArtaleMO24,
  author       = {Alessandro Artale and
                  Andrea Mazzullo and
                  Ana Ozaki},
  title        = {First-Order Temporal Logic on Finite Traces: Semantic Properties,
                  Decidable Fragments, and Applications},
  journal      = {{ACM} Trans. Comput. Log.},
  volume       = {25},
  number       = {2},
  pages        = {13:1--13:43},
  year         = {2024}
}

@article{Gabbay2005-GABMML-2,
	author = {D. M. Gabbay and A. Kurucz and F. Wolter and M. Zakharyaschev},
	journal = {Studia Logica},
	number = {1},
	pages = {147--150},
	title = {Many-Dimensional Modal Logics: Theory and Applications},
	volume = {81},
	year = {2005}
}

@article{DBLP:journals/ai/MarraDMR24,
  author       = {Giuseppe Marra and
                  Sebastijan Dumancic and
                  Robin Manhaeve and
                  Luc De Raedt},
  title        = {From statistical relational to neurosymbolic artificial intelligence:
                  {A} survey},
  journal      = {Artif. Intell.},
  volume       = {328},
  pages        = {104062},
  year         = {2024}
}

@inproceedings{DBLP:conf/nips/ManhaeveDKDR18,
  author       = {Robin Manhaeve and
                  Sebastijan Dumancic and
                  Angelika Kimmig and
                  Thomas Demeester and
                  Luc De Raedt},
  title        = {DeepProbLog: Neural Probabilistic Logic Programming},
  booktitle    = {NeurIPS},
  pages        = {3753--3763},
  year         = {2018}
}

@inproceedings{DBLP:conf/time/ApricenoPS21,
  author       = {Gianluca Apriceno and
                  Andrea Passerini and
                  Luciano Serafini},
  title        = {A Neuro-Symbolic Approach to Structured Event Recognition},
  booktitle    = {{TIME}},
  series       = {LIPIcs},
  pages        = {11:1--11:14},
  publisher    = {Schloss Dagstuhl - Leibniz-Zentrum f{\"{u}}r Informatik},
  year         = {2021}
}

@article{DBLP:journals/corr/abs-2303-17892,
  author       = {Samy Badreddine and
                  Gianluca Apriceno and
                  Andrea Passerini and
                  Luciano Serafini},
  title        = {Interval Logic Tensor Networks},
  journal      = {CoRR},
  volume       = {abs/2303.17892},
  year         = {2023}
}

@misc{collura2025absenforcingconstraintsatisfaction,
      title={ABS: Enforcing Constraint Satisfaction On Generated Sequences Via Automata-Guided Beam Search}, 
      author={Vincenzo Collura and Karim Tit and Laura Bussi and Eleonora Giunchiglia and Maxime Cordy},
      year={2025},
      eprint={2506.09701},
      archivePrefix={arXiv},
      primaryClass={cs.LG},
      url={https://arxiv.org/abs/2506.09701}, 
}

@inproceedings{DBLP:conf/ijcai/Lorello0M25,
  author       = {Luca Salvatore Lorello and
                  Marco Lippi and
                  Stefano Melacci},
  title        = {A Neuro-Symbolic Framework for Sequence Classification with Relational
                  and Temporal Knowledge},
  booktitle    = {{IJCAI}},
  pages        = {5833--5841},
  publisher    = {ijcai.org},
  year         = {2025}
}

\appendix

\section{Stable Product Configuration}\label{apd:operators}
The stable product configuration (see~\citet{serafini_garcez_2020}) ensures a fuzzy logic semantics that is well suited for gradient-descent optimization:
\begin{align*}
T(u,v) &= u\cdot v && \text{(product t-norm)}\\
N(u) &= 1-u && \text{(standard negation)}\\
S(u,v) &= u+v-uv && \text{(probabilistic sum)}\\
I(u,v) &= 1-u+uv && \text{(Reichenbach implication)}\\
A^{\mathrm{M}}_p(u_1,\dots,u_m) &= \left(\frac{1}{m}\sum_{i=1}^m u_i^p\right)^{1/p} && \text{(p-Mean)}\\
A^{\mathrm{ME}}_p(u_1,\dots,u_m) &= 1-\left(\frac{1}{m}\sum_{i=1}^m(1-u_i)^p\right)^{1/p} && \text{(p-Mean Error)}
\end{align*}

\section{Illustrative Example}\label{apd:example}
We illustrate the grounding mechanism of LTRL\textsubscript{f} by considering a car-pedestrian scenario with by the predicates $\mathsf{Car}$, $\mathsf{Ped}$, $\mathsf{Run}$ and $\mathsf{Stop}$ that describe car states, while $\mathsf{Crossing}$ and $\mathsf{OnSidewalk}$ describe pedestrian states. These are learnable predicates, i.e., their grounding is learned by a data-driven model. The predicate $\mathsf{Close}(x,y)$ relates pairs of nearby entities and is not learned. This is reasonable as cars, through their sensors, could compute the distance between themselves and other objects.

We now show the grounding computation for the formula
$
\phi: \forall x\;\forall y\;\G\bigl(\mathsf{Car}(x)\land\mathsf{Ped}(y)\land\mathsf{Crossing}(y)\land\mathsf{Close}(x,y)\to\WX\,\mathsf{Stop}(x)\bigr),
$
stating that it always holds that whenever a car $x$ is close to a crossing pedestrian $y$, the car must stop at the next time step. For the computation of $\mathcal{G_0(\phi)}$, we use the product t-norm and $p$-mean approximations with $p=2$. We consider a domain with two objects $\mathsf{c}_1$ (a car) and $\mathsf{p}_1$ (a pedestrian) over $T=5$ time steps, and assume the learned predicate groundings reported in Table~\ref{tab:Ex} (left).
\begin{table}[!htb]
\caption{Predicate groundings on the left table and pointwise evaluation of the antecedent,
consequent and implication on the right table.}
\label{tab:Ex}
    \begin{minipage}{.5\linewidth}
      \centering
\begin{tabular}{llccccc}
\toprule
& \textbf{Predicate} & $t{=}0$ & $t{=}1$ & $t{=}2$ & $t{=}3$ & $t{=}4$ \\
\midrule
$\mathsf{c}_1$ & $\mathsf{Car}$      & 0.99 & 0.99 & 0.99 & 0.99 & 0.99 \\
$\mathsf{c}_1$ & $\mathsf{Ped}$      & 0.01 & 0.01 & 0.01 & 0.01 & 0.01 \\
$\mathsf{c}_1$ & $\mathsf{Crossing}$ & 0.01 & 0.01 & 0.01 & 0.01 & 0.01 \\
$\mathsf{c}_1$ & $\mathsf{Stop}$     & 0.05 & 0.10 & 0.90 & 0.85 & 0.20 \\
\midrule
$\mathsf{p}_1$ & $\mathsf{Car}$      & 0.01 & 0.01 & 0.01 & 0.01 & 0.01 \\
$\mathsf{p}_1$ & $\mathsf{Ped}$      & 0.99 & 0.99 & 0.99 & 0.99 & 0.99 \\
$\mathsf{p}_1$ & $\mathsf{Crossing}$ & 0.10 & 0.90 & 0.80 & 0.10 & 0.05 \\
$\mathsf{p}_1$ & $\mathsf{Stop}$     & 0.05 & 0.05 & 0.05 & 0.05 & 0.05 \\
\midrule
\multicolumn{2}{l}{$\mathsf{Close}(\mathsf{c}_1,\mathsf{c}_1)$}
& 0.01 & 0.01 & 0.01 & 0.01 & 0.01 \\
\multicolumn{2}{l}{$\mathsf{Close}(\mathsf{c}_1,\mathsf{p}_1)$}
& 0.05 & 0.85 & 0.80 & 0.15 & 0.05 \\
\multicolumn{2}{l}{$\mathsf{Close}(\mathsf{p}_1,\mathsf{c}_1)$}
& 0.05 & 0.85 & 0.80 & 0.15 & 0.05 \\
\multicolumn{2}{l}{$\mathsf{Close}(\mathsf{p}_1,\mathsf{p}_1)$}
& 0.01 & 0.01 & 0.01 & 0.01 & 0.01 \\
\bottomrule
\end{tabular}
    \end{minipage}%
    \begin{minipage}{.5\linewidth}
      \centering
        \begin{tabular}{llccccc}
\toprule
\((x,y)\) & Value & \(t=0\) & \(t=1\) & \(t=2\) & \(t=3\) & \(t=4\) \\
\midrule
\multirow{3}{*}{\((\mathsf{c}_1,\mathsf{c}_1)\)}
& \(a_{x,y}(t)\) & 0.000 & 0.000 & 0.000 & 0.000 & 0.000 \\
& \(b_x(t)\)     & 0.100 & 0.900 & 0.850 & 0.200 & 1.000 \\
& \(v_{x,y}(t)\) & 1.000 & 1.000 & 1.000 & 1.000 & 1.000 \\
\midrule
\multirow{3}{*}{\((\mathsf{c}_1,\mathsf{p}_1)\)}
& \(a_{x,y}(t)\) & 0.005 & 0.750 & 0.627 & 0.015 & 0.002 \\
& \(b_x(t)\)     & 0.100 & 0.900 & 0.850 & 0.200 & 1.000 \\
& \(v_{x,y}(t)\) & 0.996 & 0.925 & 0.906 & 0.988 & 1.000 \\
\midrule
\multirow{3}{*}{\((\mathsf{p}_1,\mathsf{c}_1)\)}
& \(a_{x,y}(t)\) & 0.000 & 0.000 & 0.000 & 0.000 & 0.000 \\
& \(b_x(t)\)     & 0.050 & 0.050 & 0.050 & 0.050 & 1.000 \\
& \(v_{x,y}(t)\) & 1.000 & 1.000 & 1.000 & 1.000 & 1.000 \\
\midrule
\multirow{3}{*}{\((\mathsf{p}_1,\mathsf{p}_1)\)}
& \(a_{x,y}(t)\) & 0.000 & 0.000 & 0.000 & 0.000 & 0.000 \\
& \(b_x(t)\)     & 0.050 & 0.050 & 0.050 & 0.050 & 1.000 \\
& \(v_{x,y}(t)\) & 1.000 & 1.000 & 1.000 & 1.000 & 1.000 \\
\bottomrule
\end{tabular}
    \end{minipage} 
\end{table}
For each grounding of the pair of variables $(x,y)$ and for each time point $t$, the antecedent is computed as
$
a_{x,y}(t)
=
\mathsf{Car}(x,t)\cdot
\mathsf{Ped}(y,t)\cdot
\mathsf{Crossing}(y,t)\cdot
\mathsf{Close}(x,y,t).
$
The consequent is evaluated through the $\WX$ operator, defined as above: $b_x(t)= \mathsf{Stop}(x,t+1)$ if $t<4$, otherwise $b_x(t)= 1$ if $t=4$. Applying the
Reichenbach implication $I(a,b)=1-a+ab$, we obtain the pointwise truth value $v_{x,y}(t)=I(a_{x,y}(t),b_x(t))$.

As $\phi$ is universally quantified, we evaluate all possible couples
$(x,y)\in\{\mathsf{c}_1,\mathsf{p}_1\}^2$. Table~\ref{tab:Ex} (right)
reports the antecedent $a_{x,y}(t)$, the consequent $b_x(t)$, and the resulting
implication value $v_{x,y}(t)$ for each assignment and time step.
%
The only substantially non-vacuous assignment is $(x,y)=(\mathsf{c}_1,\mathsf{p}_1)$. At $t=1$ and $t=2$, the antecedent is high because $\mathsf{p}_1$ is crossing and close to $\mathsf{c}_1$, and the consequent is also high because $\mathsf{c}_1$ stops at the next time step. At $t=0$ and $t=3$, the antecedent is low, so the implication is almost vacuously satisfied. At $t=4$, the implication is satisfied because $\WX$ is true at the last time step.

The first aggregation is temporal: the operator $\G$ is applied independently
for each grounding of $(x,y)$ by aggregating the corresponding implication values
over time:$
s_{x,y}
=
A^{\mathrm{ME}}_2\bigl(v_{x,y}(0),\ldots,v_{x,y}(4)\bigr).
$
For the non-vacuous assignment $(x,y)=(\mathsf{c}_1,\mathsf{p}_1)$, this gives
\[
\begin{aligned}
s_{\mathsf{c}_1,\mathsf{p}_1}
&=
A^{\mathrm{ME}}_2(0.996,\,0.925,\,0.906,\,0.988,\,1.000) =
1-\left(\sum_{t=0}^{4}
(1-v_{\mathsf{c}_1,\mathsf{p}_1}(t))^2/5\right)^{1/2}
\approx 0.946.
\end{aligned}
\]
Applying the same temporal aggregation to all assignments gives $s_{x,y} = 0.964$ for the pair $(\mathsf{c}_1,\mathsf{p}_1)$ and $s_{x,y} = 0.964$ for the other three pairs. The second aggregation is first-order. The universal quantifiers aggregate the
assignment-level truth values, again using $A^{\mathrm{ME}}_2$:
\[
\begin{aligned}
\mathcal{G}_0(\phi)
&=
A^{\mathrm{ME}}_2
\bigl(
s_{\mathsf{c}_1,\mathsf{c}_1},
s_{\mathsf{c}_1,\mathsf{p}_1},
s_{\mathsf{p}_1,\mathsf{c}_1},
s_{\mathsf{p}_1,\mathsf{p}_1}
\bigr) =
A^{\mathrm{ME}}_2
\bigl(
1.000,\,
0.946,\,
1.000,\,
1.000
\bigr) \\
&=
1-\sqrt{(
(1-1.000)^2+
(1-0.946)^2+
(1-1.000)^2+
(1-1.000)^2)/4}
\approx 0.973.
\end{aligned}
\]
Thus, the grounding of the whole formula is high because the relevant car--pedestrian grounding satisfies the temporal constraint, while the other groundings are almost vacuously satisfied due to their low antecedent values.

\section{Differentiability of LTRL\textsubscript{f}}\label{apd:diff}
We show that $\mathcal{G}_t(\phi;\theta)$ is differentiable in $\theta$ for any LTRL\textsubscript{f} formula $\phi$. Here, we make the dependence on the learnable parameters $\theta$ explicit in the grounding $\mathcal{G}_t$. \begin{theorem}[Differentiability]\label{thm:diff}
Assume:
\begin{compactenum}[1.]
    \item atomic groundings for predicates and functions, $\mathcal{G}_t(P;\theta)$ and $\mathcal{G}_t(f;\theta)$, respectively, are differentiable in $\theta$;
    \item logical connectives are implemented via differentiable fuzzy operators (e.g., product t-norm);
    \item quantifiers and temporal operators ($\F$, $\G$) are implemented via $p$-mean aggregation $A_p^M$ (for $\exists$ and $\F$) and $p$-mean error aggregation $A_p^{ME}$ (for $\forall$ and $\G$) that have already been proved to be differentiable~\citep{DBLP:journals/ai/KriekenAH22}.
\end{compactenum}
Then, for any LTRL\textsubscript{f} formula $\phi$ and time $t$, the grounding $\mathcal{G}_t(\phi;\theta)$ is differentiable in $\theta$.
\end{theorem}
\begin{proof}
The proof proceeds by structural induction on the formula $\phi$.

\emph{Base case:}  
For an atomic predicate $P$ (and function $f$), $\mathcal{G}_t(P;\theta)$ (and $\mathcal{G}_t(f;\theta)$) is differentiable in $\theta$ by assumption 1.

\emph{Inductive step:}  
Assume $\mathcal{G}_t(\phi;\theta)$ and $\mathcal{G}_t(\psi;\theta)$ are differentiable. Differentiability is preserved for formulas built with:
\begin{compactitem}[$\bullet$]
    \item \emph{Logical connectives} ($\wedge$, $\vee$, $\neg$, $\rightarrow$) as the are implemented by the t-norm, t-conorm, the fuzzy negation and the fuzzy implication that are differentiable functions by assumption 2. Just as an example, when the fuzzy conjunction is implemented with the product t-norm: $\mathcal{G}_t(\phi \land \psi;\theta)    = T\bigl(\mathcal{G}_t(\phi;\theta),\mathcal{G}_t(\psi;\theta)\bigr) = \mathcal{G}_t(\phi;\theta) \cdot \mathcal{G}_t(\psi;\theta),
    $ which is differentiable by the chain rule.
    \item \emph{Quantifiers and eventually/always operators} ($\exists$, $\forall$, $\F$, $\G$) as they are implemented through differentiable aggregation operator (assumption 3) to a finite sequence of differentiable terms (objects are in a finite domain and traces are finite), hence remain differentiable.
    \item \emph{Next operators} ($\X$, $\WX$) as they are correspond to a time shift of one position. For $t < l$, differentiability follows from the inductive hypothesis applied to $t+1$, i.e., $\mathcal{G}_t(\X\phi;\theta) = \mathcal{G}_{t+1}[\phi;\theta]$; boundary values ($0$ or $1$) are constant and thus differentiable.
    \item \emph{Until} ($\U$):  
    to compute the derivative of the Until operator, we first define, for each $k \in \{t,..,l\}$, a candidate value
    \[
    c_k(\theta) = T\left(\mathcal{G}_k(\psi;\theta),A_p^{ME}(\mathcal{G}_t(\phi;\theta),\ldots,\mathcal{G}_{k-1}(\phi;\theta))\right).
    \]
    By the inductive hypothesis, all groundings $\{\mathcal{G}_i(\phi;\theta)\}_{i=t}^{k-1}$ and $\mathcal{G}_k(\psi;\theta)$ are differentiable. Since $A_p^{ME}$ the t-norm $T$ are differentiable by assumption, each candidate $c_k(\theta)$ is differentiable. The grounding of $\phi U\psi$ at each time step $t$ is then obtained by using the $p$-mean aggregation operator $A_p^M$ over all the candidates:
    \[
    \mathcal{G}_t(\phi\,\U\,\psi;\theta) = A_p^M(c_t(\theta),\dots,c_l(\theta))
    \]
    which is differentiable because $A^M_p$ is differentiable.
    \item \emph{Release} ($\R$): the proof follows from $\U$ via duality, since $\phi \,\R\, \psi$ can be expressed using $\U$ and negation, both of which are differentiable; thus $\R$ preserves differentiability.
\end{compactitem}
Given the above, all constructions involve finite compositions, sums, and products of differentiable functions over bounded time domains. As a result, $\mathcal{G}_t(\phi;\theta)$ is differentiable in $\theta$ by the chain rule.
\end{proof}

\begin{solution}
\textcolor{blue}{
\section{Notes about the Implementation}\label{apd:implementation}
The implementation is based on the LTN library~\url{https://github.com/logictensornetworks/logictensornetworks}. In the supplementary material we provide a reference implementation consisting of:
\begin{compactenum}[1.]
    \item \textbf{LTL\textsubscript{f} Parser}: A Lark-based parser that accepts first-order LTL\textsubscript{f} formulas with predicates, variables, constants, quantifiers, and temporal operators, producing a structured AST. The design of the parser was heavily inspired by the work of~\citet{DBLP:journals/procsci/DonadelloFIMM25} and their open-source repository;
    \item \textbf{Temporal Operators}: TensorFlow-based implementations of all LTRL\textsubscript{f} temporal operators ($\X$, $\WX$, $\F$, $\G$, $\U$, $\R$, and bounded variants), using LTN's built-in aggregators ($\ApM$, $\ApME$) and fuzzy connectives;
    \item \textbf{Evaluator}: An evaluator that is able to walk the AST extracted by the parser and produces the tensors, following the semantics of the operators specified in the formula;
    \item \textbf{Datasets}: the synthetic datasets and the code for reproducing the results.
\end{compactenum}
}
\end{solution}

\section{Results Plot}\label{apd:res_plot}
\begin{figure}[H]
    \centering   \includegraphics[width=1\linewidth]{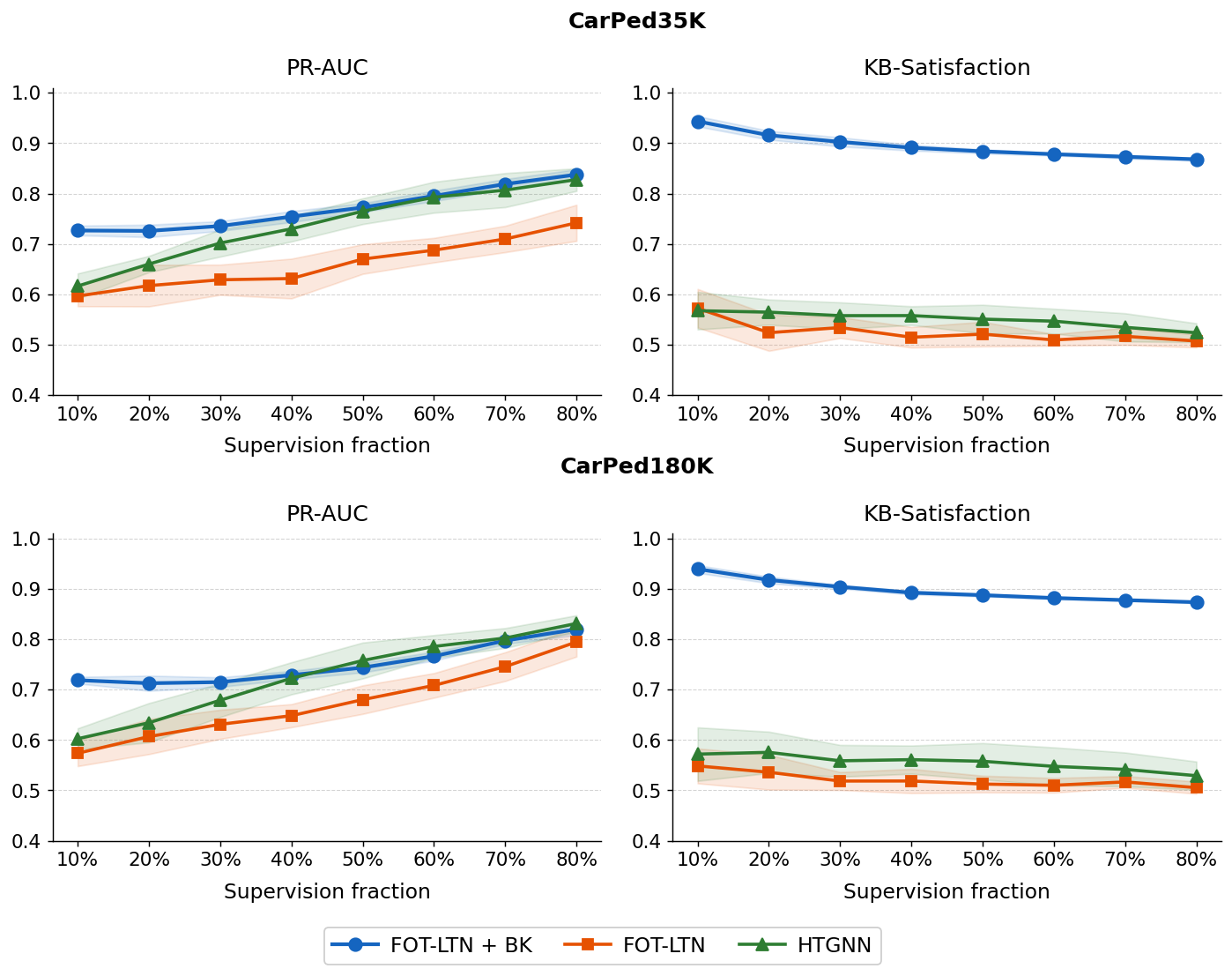}
    \caption{Average of PR-AUC and KB-Satisfaction trends along with the percentage levels of data availability for training.}
    \label{fig:results_plot2}
\end{figure}

\section{Statistical Tests}\label{apd:statistical_test}
We want to test whether our results are statistically significant, namely that it is indeed true that FOT-LTN + $\mathcal{BK}$ performs better than HTGNN. In order to do that, we use the Welch two-sample t-tests, as we want to compare two independent distributions with possibly different variance. In particular, we compare the means over the 10 reported runs on the CarPed35K and CarPed180K datasets. Since we are comparing results at multiple supervision data availability, we need to perform a correction to control for false positives. Here, we use the Benjamini-Hochberg FDR correction, and obtain the $p_{\mathrm{adj}}$-values reported in Table~\ref{tab:statistical_analysis_carped_joined}.
\begin{table}[H]
\centering
\caption{Welch t-test results comparing FOT-LTN + $\mathcal{BK}$ ($A$) against HTGNN ($B$) on the \textbf{CarPed35K} and \textbf{CarPed180K} datasets.}
\label{tab:statistical_analysis_carped_joined}
\resizebox{\textwidth}{!}{%
\begin{tabular}{lcccccc|cccccc}
\hline
\multirow{2}{*}{Fraction}
& \multicolumn{6}{c|}{\textbf{CarPed35K}}
& \multicolumn{6}{c}{\textbf{CarPed180K}} \\
\cline{2-13}
& $\bar{\mu}_A$ & $\bar{\mu}_B$ & $S_A$ & $S_B$ & $t$ & $p_{\mathrm{adj}}$
& $\bar{\mu}_A$ & $\bar{\mu}_B$ & $S_A$ & $S_B$ & $t$ & $p_{\mathrm{adj}}$ \\
\hline
10\% & 0.727 & 0.616 & 0.010 & 0.025 & 12.38 & $<0.001$
     & 0.719 & 0.603 & 0.007 & 0.021 & 16.07 & $<0.001$ \\
20\% & 0.726 & 0.660 & 0.012 & 0.016 & 9.71 & $<0.001$
     & 0.713 & 0.634 & 0.015 & 0.039 & 5.63 & $<0.001$ \\
30\% & 0.736 & 0.702 & 0.010 & 0.027 & 3.55 & 0.014
     & 0.715 & 0.679 & 0.009 & 0.033 & 3.11 & 0.028 \\
40\% & 0.754 & 0.730 & 0.011 & 0.025 & 2.65 & 0.047
     & 0.729 & 0.723 & 0.008 & 0.032 & 0.58 & 0.614 \\
50\% & 0.772 & 0.765 & 0.009 & 0.026 & 0.82 & 0.529
     & 0.744 & 0.758 & 0.010 & 0.036 & -1.14 & 0.407 \\
60\% & 0.795 & 0.793 & 0.010 & 0.031 & 0.24 & 0.811
     & 0.766 & 0.786 & 0.009 & 0.023 & -2.39 & 0.069 \\
70\% & 0.819 & 0.807 & 0.010 & 0.034 & 1.03 & 0.435
     & 0.797 & 0.802 & 0.007 & 0.020 & -0.73 & 0.548 \\
80\% & 0.838 & 0.828 & 0.011 & 0.022 & 1.21 & 0.398
     & 0.820 & 0.831 & 0.012 & 0.016 & -1.72 & 0.185 \\
\hline
\end{tabular}%
}
\end{table}
Assuming a significance level of $\alpha=0.05$, FOT-LTN + $\mathcal{BK}$ significantly outperforms HTGNN at $10\%, 20\%, 30\%, 40\%$ supervision on CarPed35K, and at $10\%, 20\%, 30\%$ supervision on CarPed180K. At higher supervision levels, the difference is not statistically significant as $p_{\mathrm{adj}} \not< 0.05$. Therefore, the test confirms that temporal axioms indeed provide higher PR-AUC in low-supervision regimes, while at medium-high supervision levels the observed differences are not statistically significant.

\end{document}